\documentclass{article}

\PassOptionsToPackage{numbers,compress}{natbib}
\usepackage{dblfloatfix}
\usepackage[preprint]{neurips_2026}

\usepackage[utf8]{inputenc} 
\usepackage[T1]{fontenc}    
\usepackage{hyperref}       
\usepackage{url}            
\usepackage{booktabs}       
\usepackage{amsfonts}       
\usepackage{amssymb}        
\usepackage{nicefrac}       
\usepackage{microtype}      
\usepackage{xcolor}         
\usepackage{amsmath}
\usepackage{algorithm}
\usepackage{algpseudocode}
\usepackage{tikz}
\usepackage{graphicx}
\usetikzlibrary{positioning,arrows.meta,shapes.geometric,fit}
\usepackage{amsthm}
\newtheorem{theorem}{Theorem}[section]
\newtheorem{lemma}[theorem]{Lemma}

\theoremstyle{definition}

\theoremstyle{remark}

\title{LATTICE: Graph Self-Supervised Learning for Multimodal Spatial Omics Integration}

\author{%
  Jagan Mohan Reddy Dwarampudi \\
  University of Houston \\
  \texttt{jdwaramp@cougarnet.uh.edu} \\
  \And
  Veena Kochat \\
  MD Anderson Cancer Center \\
  \texttt{vkochat@mdanderson.org} \\
  \And
  Suresh Satpati \\
  MD Anderson Cancer Center \\
  \texttt{ssatpati@mdanderson.org} \\
  \And
  Hien Van Nguyen \\
  University of Houston \\
  \texttt{hvnguy35@central.uh.edu} \\
  \And
  Kunal Rai \\
  MD Anderson Cancer Center \\
  \texttt{krai@mdanderson.org} \\
  \And
  Tania Banerjee \\
  University of Houston \\
  \texttt{tbanerje@central.uh.edu} \\
}

\begin{document}

\setcounter{topnumber}{8}
\setcounter{bottomnumber}{8}
\setcounter{totalnumber}{18}
\renewcommand{\topfraction}{0.90}
\renewcommand{\bottomfraction}{0.85}
\renewcommand{\textfraction}{0.08}
\renewcommand{\floatpagefraction}{0.72}
\renewcommand{\dbltopfraction}{0.90}
\renewcommand{\dblfloatpagefraction}{0.72}

\maketitle

\begin{abstract}
  Spatially resolved omics studies increasingly combine transcriptomic and epigenomic assays, yet downstream analysis is often still performed using single-modality pipelines. We present LATTICE (Latent Alignment of Tissue-level and Transcriptomic Information for Cross-modal Embedding), a graph-based self-supervised framework that learns spot-level representations from harmonized multimodal features. LATTICE integrates five aligned modality blocks per Visium spot: Visium RNA, scMultiome RNA, scMultiome ATAC, spatial ATAC, and spatial CUT\&Tag. These modalities capture spatial transcriptomic measurements, single-cell inferred regulatory activity, and in situ chromatin and histone states within a unified lattice representation. LATTICE constructs a spatial neighborhood graph and trains a TransformerConv encoder using masked reconstruction, cross-modal alignment, and spatial smoothness objectives. On a private 11-sample melanoma cohort from an anonymized clinical collaborator comprising 54{,}912 total spots, LATTICE demonstrated stable optimization behavior, reproducible embeddings across analysis seeds, and complete multimodal integration across all samples. Adding scMultiome RNA to Visium RNA alone substantially improved concordance with Space Ranger clusters across 11 runs (adjusted Rand index [ARI] +0.157, normalized mutual information [NMI] +0.143, and spatial contiguity +0.174). Additional modalities further improved spatial contiguity and multimodal utility score (MUS), although they sometimes reduced agreement with RNA-derived reference labels, likely because the learned embeddings captured chromatin and regulatory structure beyond transcriptomic similarity alone. These results position LATTICE as a practical and empirically grounded framework for multimodal spatial omics integration, while also highlighting the need for stronger supervision and broader external benchmarking.
\end{abstract}

\section{Introduction}
\label{sec:introduction}

Spatially resolved omics technologies have transformed our ability to study tissue organization by measuring molecular profiles while preserving spatial context. Platforms such as spatial transcriptomics~\cite{staahl2016visualization} enable genome-wide gene expression across sections, while spatial epigenomic assays and single-cell multiome profiles add accessibility, histone, and cell-state-resolved chromatin readouts. In integrative cohorts, those channels are not redundant. Visium RNA captures broad transcriptional organization. Mapped scMultiome provides higher-resolution transcriptional signals and ATAC-derived gene activities mapped onto the same spots. In situ spatial ATAC and spatial CUT\&Tag~\cite{kaya2019cut} add section-level chromatin and histone-modification context that need not follow RNA-defined domains. An expression-only view can therefore conflate cell state, microenvironment, and regulatory state. Integrating these sources offers spatially grounded views of both transcriptional and regulatory organization.

Despite rapid progress, computational integration of multimodal spatial data remains challenging. Existing approaches often treat one aspect of the problem in isolation. GraphST~\cite{long2023spatially} and related models use graph neural networks and self-supervision on expression and spatial proximity. Other work maps single-cell profiles to locations or aligns modalities through shared latent spaces, but typically does not learn a unified spot-level representation from a fully aligned multimodal feature tensor.

A key limitation is the lack of frameworks that simultaneously capture spatial structure and cross-modal consistency in multimodal spatial datasets. In practice, such datasets are often assembled by combining spatial transcriptomics with projected single-cell modalities or spatial epigenomic measurements. However, downstream analysis is typically performed using methods designed for single-modality data, limiting the ability to fully exploit complementary signals across modalities. This motivates the need for representation learning approaches that operate directly on multimodal spatial feature graphs.

Figure~\ref{fig:lattice_overview} summarizes the overall LATTICE workflow and multimodal integration strategy.
Our evaluation uses harmonized spot-level tensors generated from two upstream components that we reference throughout the paper. ReCAST is an internal engineering pipeline that standardizes inputs, performs multimodal harmonization, and applies sample-level quality control; implementation and export details are summarized in Appendix~\ref{sec:appendix_recast}. SARSIM (Spatially Anchored Regulatory State Inference in Melanoma)~\cite{dwarampudi2026spatially} is a framework for spatially anchored regulatory inference that integrates paired Visium and single-cell multiome data, learns a soft cell-to-spot mapping, and projects accessibility and motif activity into tissue space. We additionally reuse its overlap-gene constraints and clustering metadata during evaluation. Section~\ref{sec:exp_setup} states how their outputs become the five modality blocks fed to LATTICE.

\begin{figure*}[!tbp]
\footnotesize
\centering
\includegraphics[width=\columnwidth, trim=49cm 57cm 42cm 52cm]{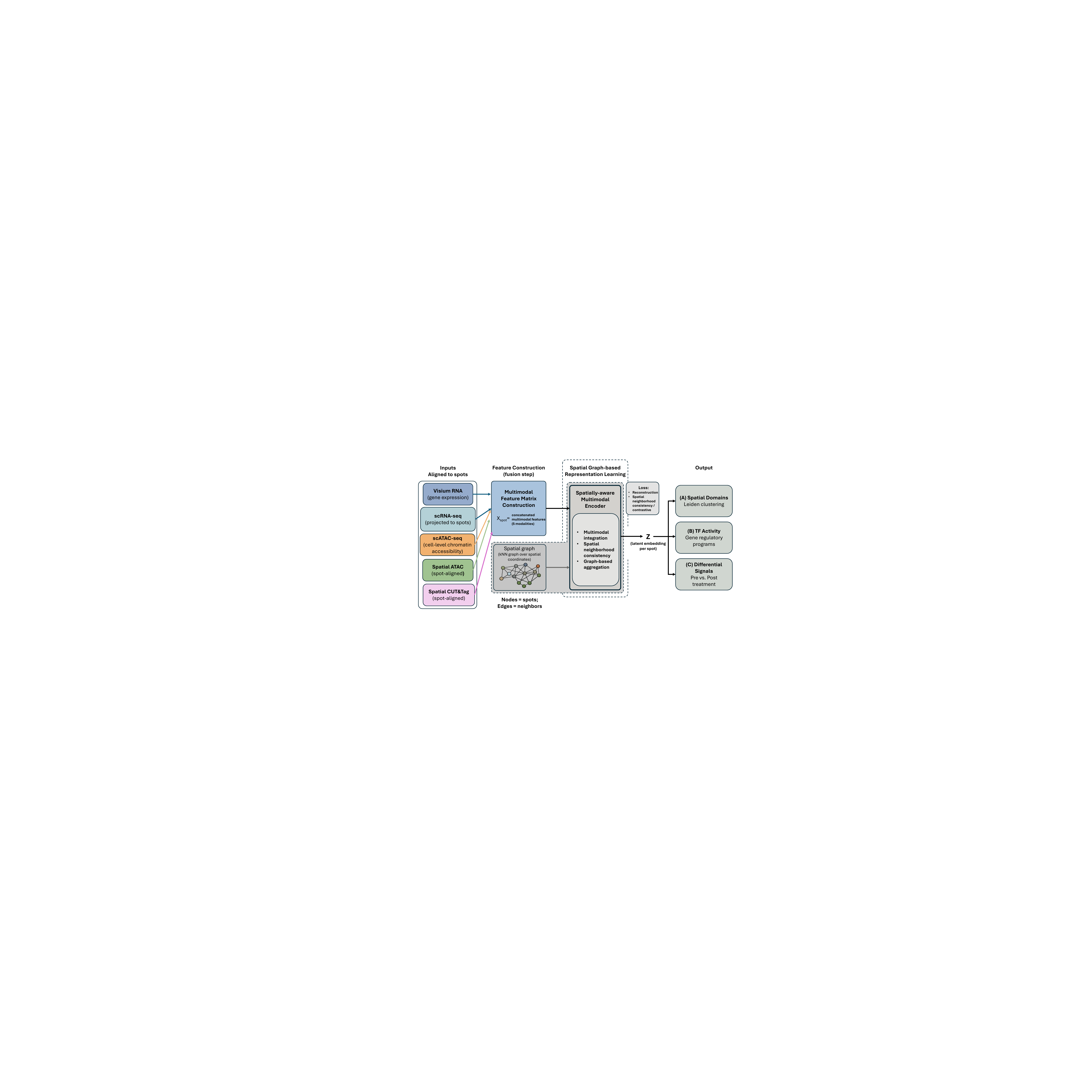}
\caption{
Overview of LATTICE. Multimodal spatial features combining transcriptomic and epigenomic signals aligned at Visium spot resolution, are fused into a unified feature matrix and organized using a spatial k-nearest neighbor graph. A graph-based multimodal encoder learns spot-level embeddings through masked reconstruction, cross-modal alignment, and spatial smoothness objectives. The resulting embeddings capture both molecular and spatial structure and enable downstream tasks including spatial domain identification, regulatory program analysis, and pre/post-treatment differential characterization.
}
\label{fig:lattice_overview}
\end{figure*}

In this work, we present LATTICE, a graph-based self-supervised learning framework for multimodal spatial omics integration. LATTICE takes as input spatially organized multimodal features that combine transcriptomic and epigenomic signals, and constructs a graph that encodes spatial proximity and feature similarity. We then learn latent representations using a graph neural network trained with a combination of self-supervised objectives. Specifically, LATTICE integrates (i) masked feature reconstruction to preserve modality-specific information, (ii) cross-modal alignment to enforce consistency across modalities, and (iii) spatial regularization to enforce local coherence in tissue structure.

We evaluate LATTICE on a private melanoma multimodal cohort from an anonymized clinical collaborator and report both per-sample and joint analyses. We define a cumulative modality ladder referred to throughout the paper as M1–M5:
\begin{itemize}
\item \textbf{M1:} Visium RNA only.
\item \textbf{M2:} M1 plus projected scMultiome RNA.
\item \textbf{M3:} M2 plus projected scMultiome ATAC gene scores.
\item \textbf{M4:} M3 plus spatial ATAC.
\item \textbf{M5:} M4 plus spatial CUT\&Tag (full input tensor).
\end{itemize}
Levels M2 to M3 align with the SARSIM-scope inputs, and levels M4 to M5 add ReCAST-derived in situ epigenomics (Section~\ref{sec:exp_main}). LATTICE is modular with respect to multimodal feature construction, which makes it compatible with a range of upstream integration or mapping strategies.

In summary, our contributions are as follows:
\begin{enumerate}
\item We introduce LATTICE, a graph self-supervised learning framework for learning representations from multimodal spatial omics data.
\item We propose a combination of masked reconstruction, cross-modal alignment losses, and spatial regularization for multimodal spatial representation learning.
\item We provide an end-to-end implementation that harmonizes five modality blocks at spot resolution and exports reproducible integration/training artifacts.
\item We quantify modality-ladder behavior on 11 successful cohort runs, identifying both robust gains and challenging integration regimes that require additional methodological work.
\end{enumerate}
These results highlight the potential of graph-based self-supervised learning as a general paradigm for multimodal spatial omics analysis.

\section{Related Work}

\paragraph{Spatial omics, multimodal integration, and graph self-supervision}
Spatial transcriptomics methods increasingly use graphs and self-supervision. GraphST~\cite{long2023spatially}, STAGATE~\cite{dong2022deciphering}, and SpaGCN~\cite{hu2021spagcn} combine expression with spatial neighborhoods for domains and clustering, but typically emphasize RNA with limited multimodal side information. Separately, single-cell spatial integration methods (SIMO~\cite{yang2025spatial}, SSpMosaic~\cite{zhang2026robust}, MaxFuse~\cite{zhu2024maxfuse}) match or align dissociated cells to tissue, yielding maps or fused views rather than a unified encoder over spot-aligned multimodal tensors. SARSIM~\cite{dwarampudi2026spatially} extends this line by integrating Visium with scMultiome to learn spatially coherent cell-to-spot mappings and project accessibility and motif activity into tissue space for domain-level regulatory interpretation. Graph self-supervised learning is well developed for attributed graphs but is still used predominantly in single- or weakly multimodal settings. Multimodal spatial analysis therefore needs objectives that couple spatial neighborhoods with heterogeneous spot features, not only cross-dataset alignment.

\paragraph{Positioning of LATTICE}
LATTICE targets representation learning on a single spatial graph whose node attributes concatenate all harmonized assay blocks at each location. Masked reconstruction preserves per-modality signal, cross-modal alignment ties views of the same spot, and spatial regularization enforces local coherence. Together these components yield one embedding space for downstream spatial and regulatory analysis without requiring a separate alignment stage. Relative to the modality ladder introduced in Section~\ref{sec:introduction}, M2 through M3 correspond to the Visium plus projected scMultiome slice of the upstream stack, which is the assay pairing SARSIM is built around, and M4 through M5 add ReCAST-derived in situ chromatin channels on the same spot lattice (Section~\ref{sec:exp_setup}).

\section{Proposed Method: LATTICE}
\label{sec:method}

\subsection{Multimodal inputs and objective overview}

We propose LATTICE, a graph-based self-supervised learning framework for learning representations from multimodal spatial omics data. Given spatially resolved measurements augmented with multiple molecular modalities, LATTICE learns low-dimensional embeddings that jointly capture spatial structure and cross-modal relationships. Let $N$ denote the number of spatial locations (e.g., spots), and let
\begin{equation}
X = [X^{(1)}, X^{(2)}, \dots, X^{(B)}] \in \mathbb{R}^{N \times D}
\end{equation}
be the concatenated multimodal feature matrix, where $X^{(b)} \in \mathbb{R}^{N \times d_b}$ corresponds to modality block $b$, $d_b$ is the feature dimension of block $b$, and $D = \sum_{b=1}^{B} d_b$ is the total input feature dimension. In our cohort $B=5$, matching levels M1 to M5 in the modality ladder. Visium and projected scMultiome contribute gene-expression and ATAC-derived gene-score features, whereas spatial ATAC and spatial CUT\&Tag~\cite{kaya2019cut} provide in situ chromatin and histone measurements on the same spot lattice. We treat these blocks as complementary views of each location for reconstruction and cross-modal alignment (block listing in Section~\ref{sec:exp_setup}). LATTICE then learns node embeddings $Z \in \mathbb{R}^{N \times d}$, where $d$ is the embedding dimension, using three self-supervised objectives: masked reconstruction $\mathcal{L}_{\mathrm{rec}}$ (Eq.~\ref{eq:lrec}), cross-modal alignment $\mathcal{L}_{\mathrm{align}}$ (Eq.~\ref{eq:lalign}), and spatial regularization $\mathcal{L}_{\mathrm{spatial}}$ (Eq.~\ref{eq:lspatial}). Their weighted sum is defined in Eq.~\ref{eq:ltotal}, and implementation details appear in Section~\ref{sec:implementation}.

\subsection{Spatial graph construction}

Each node $i \in V$ is a Visium spot in the concatenated multimodal feature space, with spatial coordinate $s_i \in \mathbb{R}^2$ and multimodal feature vector $x_i \in \mathbb{R}^D$. We define edges by spatial proximity. For each node $i$, we define its neighborhood $\mathcal{N}(i)$ as the set of its $k$-nearest neighbors (kNN) based on Euclidean distance:
\begin{equation}
\mathcal{N}(i) = \text{kNN}(s_i)
\end{equation}
Edges are added between $i$ and all $j \in \mathcal{N}(i)$. Optionally, edge weights are defined using a Gaussian kernel:
\begin{equation}
w_{ij} = \exp\left(-\frac{\|s_i - s_j\|^2}{2\sigma^2}\right)
\end{equation}

\subsection{Graph encoder and self-supervised objective}

We employ a graph neural network to learn node embeddings from the multimodal feature graph. Given input features $X$ and graph $G$, the encoder produces embeddings $Z = f_\theta(X, G)$. The implementation stacks TransformerConv message-passing layers with modality-specific input adapters and LayerNorm using PyTorch~\cite{paszke2019pytorch} and PyTorch Geometric~\cite{fey2019fast}. Full architectural and optimization hyperparameters appear in Appendix~\ref{sec:implementation}. To preserve modality-specific information, we adopt a masked reconstruction objective. Let $\Omega \in \{0,1\}^{N \times D}$ be the reconstruction mask, where $\Omega_{ij}=1$ means that entry $(i,j)$ is hidden from the encoder and used as a prediction target. The masked input is:
\begin{equation}
\tilde{X} = (1-\Omega) \odot X
\end{equation}
A decoder $g_\phi$ maps embeddings back to feature space:
\begin{equation}
\hat{X} = g_\phi(Z)
\end{equation}
We compute reconstruction loss only on hidden entries:
\begin{equation}
\label{eq:lrec}
\mathcal{L}_{\text{rec}} = \frac{1}{|\Omega|} \sum_{i,j} \Omega_{ij} (X_{ij} - \hat{X}_{ij})^2
\end{equation}
where $|\Omega|=\sum_{i,j}\Omega_{ij}$ is the number of masked target entries.

To align representations across modalities, we introduce a cross-modal alignment objective. For modalities $a$ and $b$, we treat modality-specific features $x_i^{(a)}$ and $x_i^{(b)}$ from the same spot as a positive pair, and features from different spots are negatives. We define modality-specific projections:
\begin{equation}
h_i^{(m)} = p_m(z_i^{(m)})
\end{equation}
where $p_m(\cdot)$ is a learnable projection for modality $m$. The alignment loss is a noise-contrastive style objective~\cite{gutmann2010noise}:
\begin{equation}
\label{eq:lalign}
\mathcal{L}_{\text{align}} = - \sum_{i=1}^{N} \log \frac{\exp(\text{sim}(h_i^{(a)}, h_i^{(b)}) / \tau)}{\sum_{j=1}^{N} \exp(\text{sim}(h_i^{(a)}, h_j^{(b)}) / \tau)}
\end{equation}
where $\text{sim}(\cdot, \cdot)$ denotes cosine similarity and $\tau$ is a temperature parameter. To enforce spatial smoothness, we encourage neighboring nodes to have similar embeddings:
\begin{equation}
\label{eq:lspatial}
\mathcal{L}_{\text{spatial}} = \sum_{(i,j) \in E} w_{ij} \| z_i - z_j \|_2^2
\end{equation}

The final objective combines all components:
\begin{equation}
\label{eq:ltotal}
\mathcal{L} = \lambda_1 \mathcal{L}_{\text{rec}} + \lambda_2 \mathcal{L}_{\text{align}} + \lambda_3 \mathcal{L}_{\text{spatial}}
\end{equation}
where $\lambda_1, \lambda_2, \lambda_3$ are hyperparameters.

\subsection{Training and inference}

Training first builds the spatial $k$-NN graph. At each epoch,
LATTICE samples a reconstruction mask, runs the encoder and
decoder, and minimizes the objective $\mathcal{L}$ from
Eq.~\ref{eq:ltotal} with AdamW~\cite{loshchilov2017decoupled}
until early stopping. Algorithm~\ref{alg:lattice} summarizes
the training procedure.

\begin{algorithm}[t]
\caption{Training LATTICE}
\label{alg:lattice}
\footnotesize
\begin{algorithmic}[1]
\Require Multimodal feature matrix $X=[X^{(1)},\dots,X^{(B)}]$, spatial coordinates $S$, graph neighborhood size $k$, encoder $f_\theta$, decoder $g_\phi$, modality projectors $\{p_b\}_{b=1}^B$
\Ensure Spatial embeddings $Z$

\State Build spatial graph $G=(V,E)$ from $k$-nearest neighbors of $S$
\For{each epoch}
    \State Sample reconstruction mask $\Omega$
    \State Mask input features: $\tilde{X} \gets (1-\Omega) \odot X$
    \State Encode graph nodes: $Z \gets f_\theta(\tilde{X}, G)$
    \State Decode features: $\hat{X} \gets g_\phi(Z)$
    \State Compute masked reconstruction loss $\mathcal{L}_{\mathrm{rec}}$
    \State Compute cross-modal alignment loss $\mathcal{L}_{\mathrm{align}}$ across modality pairs
    \State Compute spatial smoothness loss $\mathcal{L}_{\mathrm{spatial}}$
    \State Form total loss:
    \[
    \mathcal{L} =
    \lambda_1 \mathcal{L}_{\mathrm{rec}} +
    \lambda_2 \mathcal{L}_{\mathrm{align}} +
    \lambda_3 \mathcal{L}_{\mathrm{spatial}}
    \]
    \State Update model parameters by backpropagation
\EndFor
\State Compute final embeddings with trained encoder: $Z \gets f_\theta(X, G)$
\State \Return $Z$
\end{algorithmic}
\end{algorithm}

After training, LATTICE produces embeddings $Z$ for each spatial location. These embeddings can be used for downstream tasks such as spatial domain identification, visualization, and integration across datasets. LATTICE is not tied to a specific upstream projection tool or spatial assay. More generally, LATTICE is compatible with any upstream pipeline that produces modality-specific feature blocks aligned to a common spatial unit, including alternatives to SARSIM or ReCAST, making the framework extensible across integration strategies, epigenomic assays, and emerging high-resolution spatial platforms such as Visium HD~\cite{tenx_visium_hd} and Xenium~\cite{tenx_xenium}.

\section{Experiments}
\label{sec:experiments}

Our experiments address three questions:
\begin{enumerate}
\item How does embedding quality change across modality-ladder levels M1 to M5?
\item Do chromatin and histone features improve spatial organization when agreement with an RNA-only reference decreases?
\item Which training components contribute most to the observed gains?
\end{enumerate}

\subsection{Dataset and evaluation setup}
\label{sec:exp_setup}

Table~\ref{tab:processed_samples} summarizes the retained cohort. We start from 14 source samples and retain 11 after quality control, giving 54{,}912 Visium spots. Each retained slide has 4,992 spots. Three patients have paired pre/post-treatment samples. These paired samples capture tissue before and after treatment.

For every spot, LATTICE assembles five aligned blocks: Visium RNA, projected scMultiome RNA, projected scMultiome ATAC gene scores, spatial ATAC, and spatial CUT\&Tag. We evaluate these blocks using the modality ladder levels M1 to M5 defined in Section~\ref{sec:introduction}. ReCAST performs preprocessing and harmonization, SARSIM performs spatially anchored regulatory projection, and LATTICE performs graph SSL representation learning.

\begin{table*}[!bt]
\centering
\caption{Processed cohort used for evaluation. Each retained slide has 4,992 Visium spots. Signal columns report mean post-harmonization values per spot.}
\label{tab:processed_samples}
\resizebox{\textwidth}{!}{
\begin{tabular}{lcccc}
\toprule
\textbf{Processed sample} & \textbf{Multiome cells} & \textbf{Spatial ATAC signal} & \textbf{Spatial CUT\&Tag signal} & \textbf{Final intersected genes} \\
\midrule
1-pre & 9,531 & 2.415 & 0.904 & 248 \\
1-post & 16,936 & 1.934 & 1.547 & 198 \\
2-post & 746 & 1.352 & 1.185 & 150 \\
4-pre & 16,391 & 3.426 & 1.443 & 133 \\
4-post & 15,577 & 3.918 & 1.297 & 322 \\
5-pre & 8,005 & 2.188 & 2.788 & 154 \\
5-post & 10,394 & 2.051 & 2.268 & 129 \\
7-pre & 7,734 & 2.843 & 1.796 & 146 \\
8-pre & 716 & 5.819 & 1.091 & 211 \\
9-pre & 9,791 & 2.532 & 1.105 & 205 \\
10-post & 3,254 & 2.505 & 1.836 & 160 \\
\bottomrule
\end{tabular}
}
\vspace{-4mm}
\end{table*}

ARI and NMI~\cite{strehl2002cluster} compare LATTICE clusters with Space Ranger RNA-derived clusters. Space Ranger is the 10x Genomics analysis pipeline for Visium data~\cite{tenx_spaceranger}, and its RNA-derived clusters are used here as a transcriptomic reference, not ground truth. Spatial contiguity measures whether neighboring spots receive coherent labels, and silhouette measures cluster separation in embedding space.

Because ARI and NMI reward agreement with an RNA-only reference, they can undervalue embeddings that capture chromatin or histone structure beyond transcriptomic similarity. We therefore define a multimodal utility score (MUS) to summarize complementary aspects of spatial and multimodal organization. MUS averages four normalized quantities: spatial contiguity, embedding silhouette, same-cluster spatial-neighbor fraction, and overlap between embedding $k$-nearest neighbors and spatial $k$-nearest neighbors:
\begin{equation}
\mathrm{MUS}_v
=
\frac{1}{4}
\left(
\mathrm{SpotCut}^{\mathrm{norm}}_v
+
\mathrm{Silhouette}^{\mathrm{norm}}_v
+
\mathrm{BioKNN}^{\mathrm{norm}}_v
+
\mathrm{BioJaccard}^{\mathrm{norm}}_v
\right).
\end{equation}
Higher MUS values indicate embeddings that jointly preserve spatial coherence, local neighborhood consistency, and multimodal biological structure. Full normalization details appear in Appendix~\ref{sec:a2}.

Table~\ref{tab:cohort_benchmark_mus} compares LATTICE with spatial and multimodal baselines under the same evaluation metrics. When a baseline has multiple configurations, we report the strongest configuration according to cohort-mean ARI and leave sweep details to the appendix. The raw data cannot be publicly released because of institutional data-use and privacy restrictions. Institution and collaborator identifiers are anonymized for double-blind review and can be restored in the camera-ready version where appropriate.

\subsection{Main results}
\label{sec:exp_main}

\begin{table*}[!tp]
\centering
\caption{Cohort-level benchmarks and LATTICE modality ladder. Modalities are Visium RNA only (M1), plus projected scMultiome RNA (M2), plus scMultiome ATAC gene scores (M3), plus spatial ATAC (M4), and plus spatial CUT\&Tag (M5). ARI/NMI compare with Space Ranger RNA clusters. MUS summarizes spatial and multimodal structure.}
\label{tab:cohort_benchmark_mus}
\resizebox{\textwidth}{!}{
\begin{tabular}{llcc|ccc}
\toprule
\textbf{Method} & \textbf{Modalities} & \textbf{ARI} & \textbf{NMI} & \textbf{Spatial contiguity} & \textbf{Silhouette} & \textbf{MUS} \\
\midrule
GraphST & M1 & $0.423 \pm 0.124$ & $0.536 \pm 0.107$ & $0.837 \pm 0.052$ & $0.055 \pm 0.061$ & $0.666$ \\
STAGATE & M1 & $0.308 \pm 0.136$ & $0.420 \pm 0.130$ & $0.825 \pm 0.061$ & $0.005 \pm 0.069$ & $0.657$ \\
SpaGCN & M1 + histology & $0.408 \pm 0.078$ & $0.553 \pm 0.098$ & $0.707 \pm 0.065$ & $0.426 \pm 0.063$ & $0.387$ \\
SIMO & M1 & $0.358 \pm 0.155$ & $0.457 \pm 0.135$ & $0.697 \pm 0.090$ & $0.064 \pm 0.056$ & $0.139$ \\
MaxFuse & scRNA and spatial ATAC & $0.278 \pm 0.126$ & $0.442 \pm 0.107$ & $0.794 \pm 0.069$ & $0.199 \pm 0.106$ & $0.493$ \\
\midrule
LATTICE & M1 & $0.269 \pm 0.067$ & $0.364 \pm 0.066$ & $0.653 \pm 0.135$ & $0.147 \pm 0.176$ & $0.111$ \\
LATTICE & M2 & $0.426 \pm 0.075$ & $0.507 \pm 0.073$ & $0.827 \pm 0.043$ & $0.440 \pm 0.122$ & $0.733$ \\
LATTICE & M3 & $0.417 \pm 0.082$ & $0.493 \pm 0.077$ & $0.824 \pm 0.053$ & $0.426 \pm 0.137$ & $0.718$ \\
LATTICE & M4 & $0.343 \pm 0.086$ & $0.462 \pm 0.101$ & $0.844 \pm 0.037$ & $\mathbf{0.453 \pm 0.130}$ & $0.801$ \\
\textbf{LATTICE} & \textbf{M5} & $0.329 \pm 0.115$ & $0.450 \pm 0.122$ & $\mathbf{0.850 \pm 0.036}$ & $0.417 \pm 0.088$ & $\mathbf{0.803}$ \\
\bottomrule
\end{tabular}
}
\end{table*}

\begin{figure*}[!tbp]
    \centering
    \includegraphics[width=\textwidth]{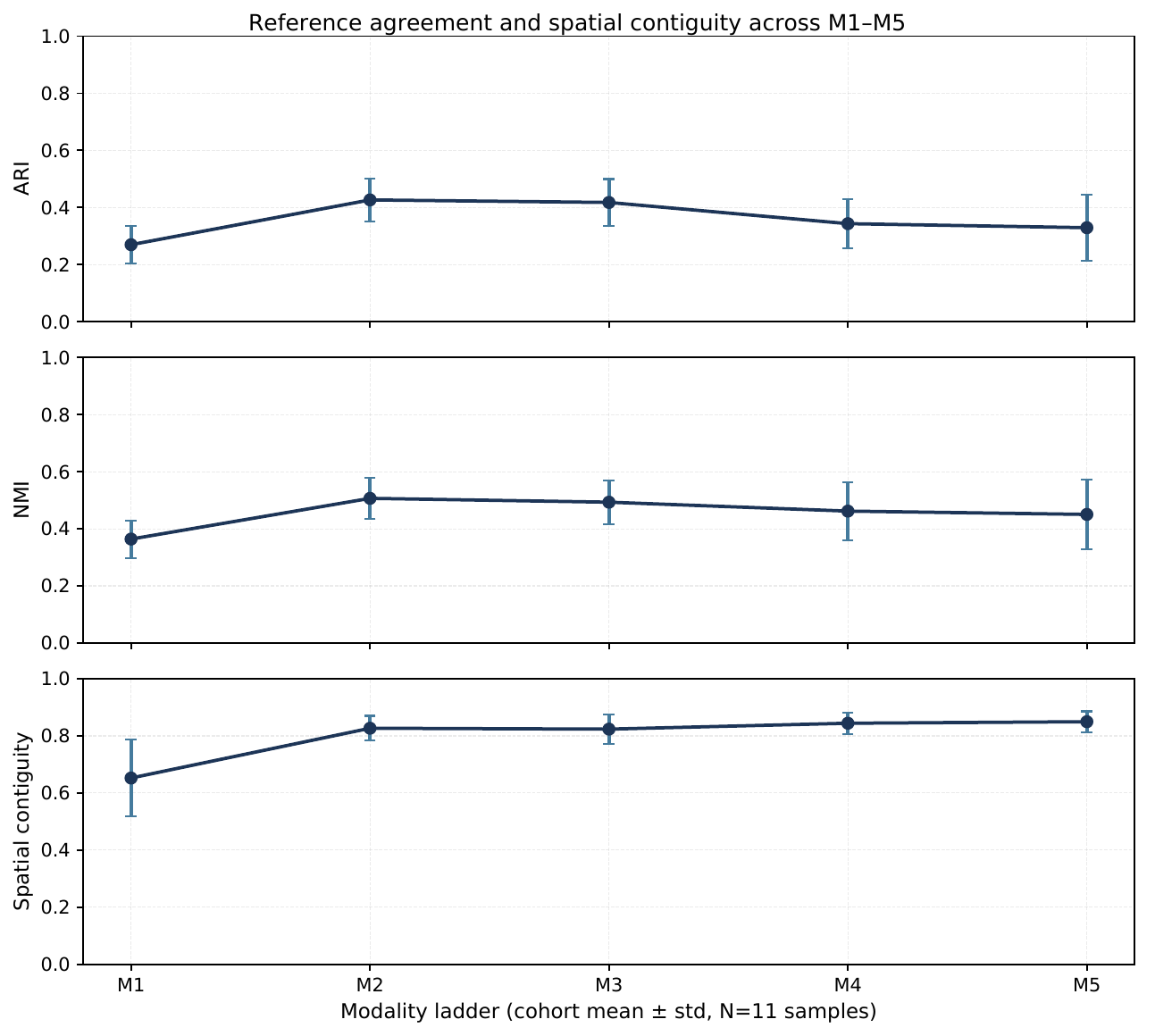}
    \caption{M1 through M5 trends from Table~\ref{tab:cohort_benchmark_mus}. RNA-reference agreement peaks near M2 through M3 (SARSIM-scope inputs). ReCAST-backed M4 through M5 favor contiguity and MUS.}
    \label{fig:modality_ladder}
\end{figure*}

Table~\ref{tab:cohort_benchmark_mus} gives the main cohort-level result, and Figure~\ref{fig:modality_ladder} visualizes cohort-level trends across the M1–M5 modality ladder.
\begin{itemize}
\item The largest change is from M1 to M2. Adding projected scMultiome RNA to Visium RNA improves ARI, NMI, spatial contiguity, silhouette, and MUS.
\item Adding scMultiome ATAC gene scores at M3 keeps performance close to M2.
\item Adding in situ chromatin and histone channels at M4 and M5 changes the behavior. ARI and NMI decrease relative to M2 because the embedding no longer follows the RNA reference alone, while spatial contiguity and MUS increase.
\end{itemize}

LATTICE M5 has the highest MUS, with M4 close behind. LATTICE M1 underperforms RNA-specialized baselines because the architecture and objectives are designed for multimodal representation learning rather than optimization for transcriptomic clustering alone. This supports the claim that full multimodal inputs improve spatial coherence and multimodal neighborhood structure, even when RNA-reference agreement is not maximal.

Figure~\ref{fig:joint_embedding} shows the joint PCA for the three paired patients. Pre- and post- treatment samples separate within each patient, but the shift size differs by patient. This motivates considering pre/post-treatment behavior as patient-specific rather than as a universal treatment signature.

Figures~\ref{fig:spatial_clusters}--\ref{fig:accessibility_gene_links} show how the embeddings support downstream spatial and regulatory analysis for one paired patient: spatial domains remain coherent on tissue coordinates, TF programs vary across Leiden clusters, and accessibility--gene links can be summarized from the learned representation.

\begin{figure}[!tbp]
    \centering
    \includegraphics[width=0.32\linewidth]{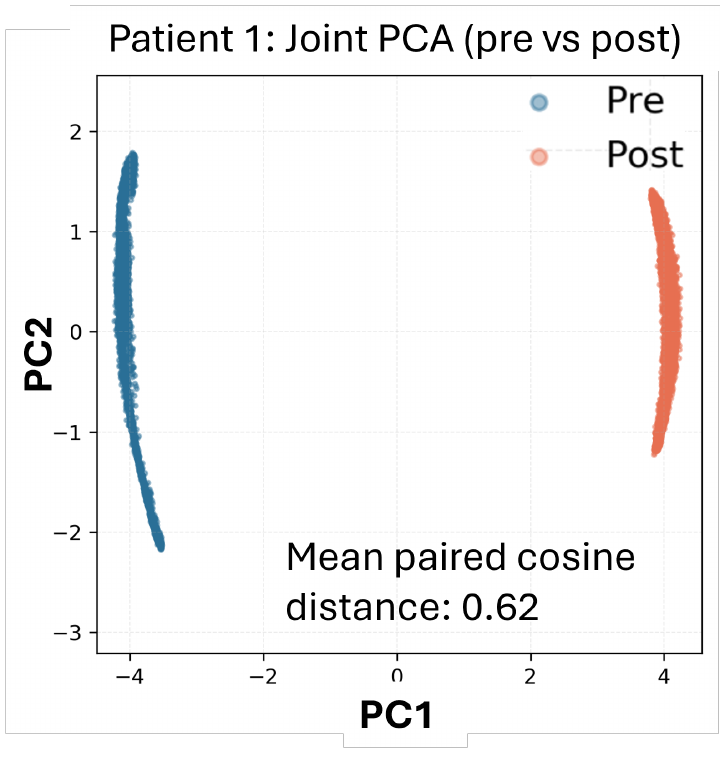}
    \hfill
    \includegraphics[width=0.32\linewidth]{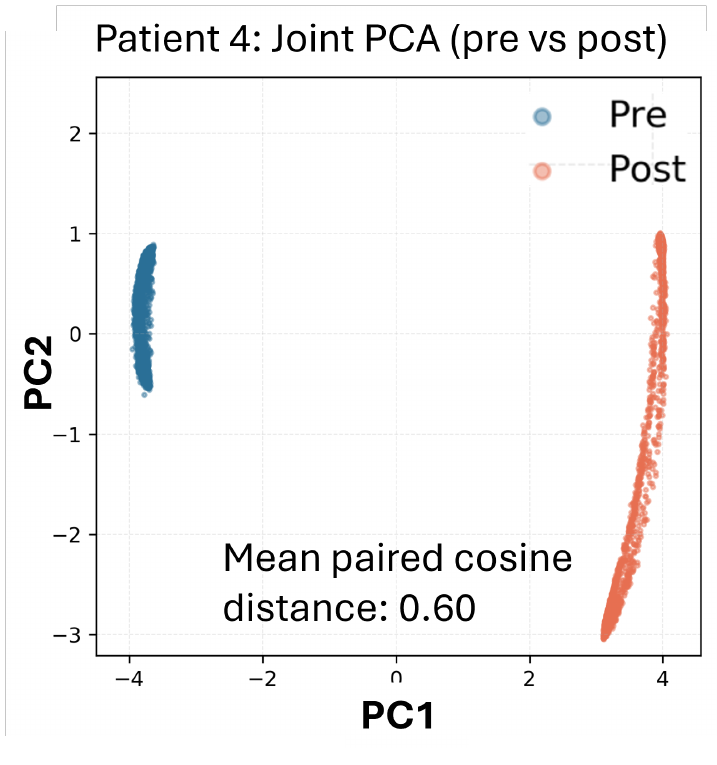}
    \hfill
    \includegraphics[width=0.32\linewidth]{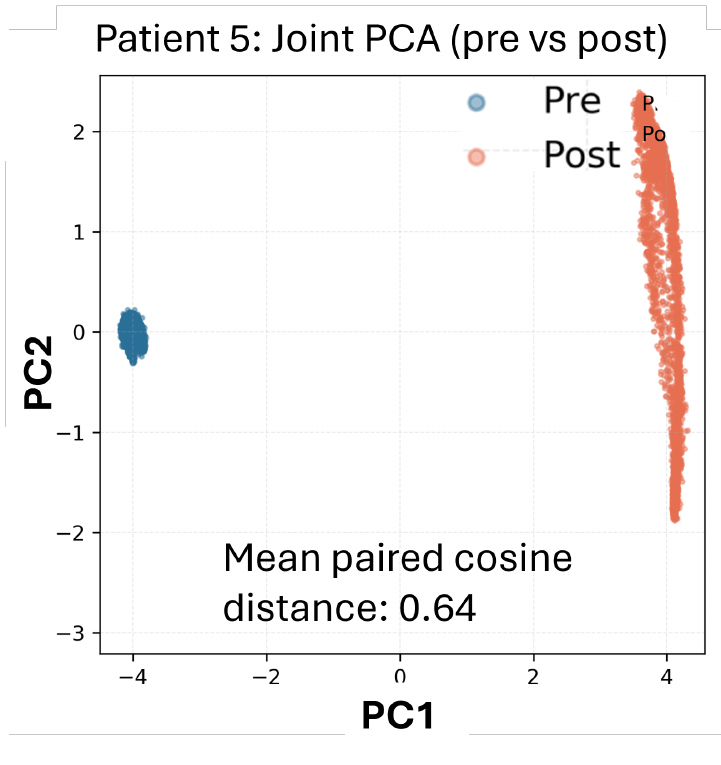}
    \caption{Joint PCA for paired pre/post-treatment samples. Shift magnitude varies by patient.}
    \label{fig:joint_embedding}
\end{figure}

\begin{figure}[!b]
    \centering
    \hfill
    \includegraphics[width=0.35\linewidth]{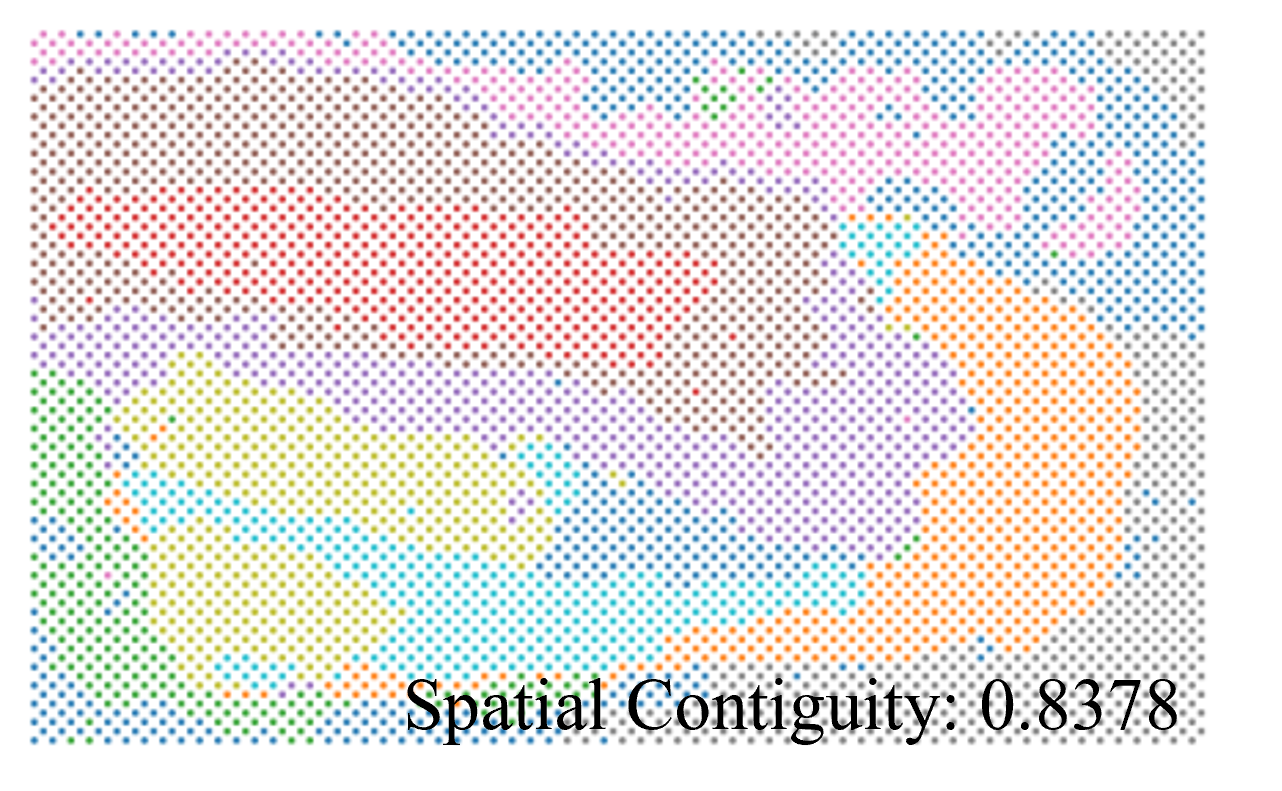}
    \hfill
    \hfill
    \includegraphics[width=0.35\linewidth]{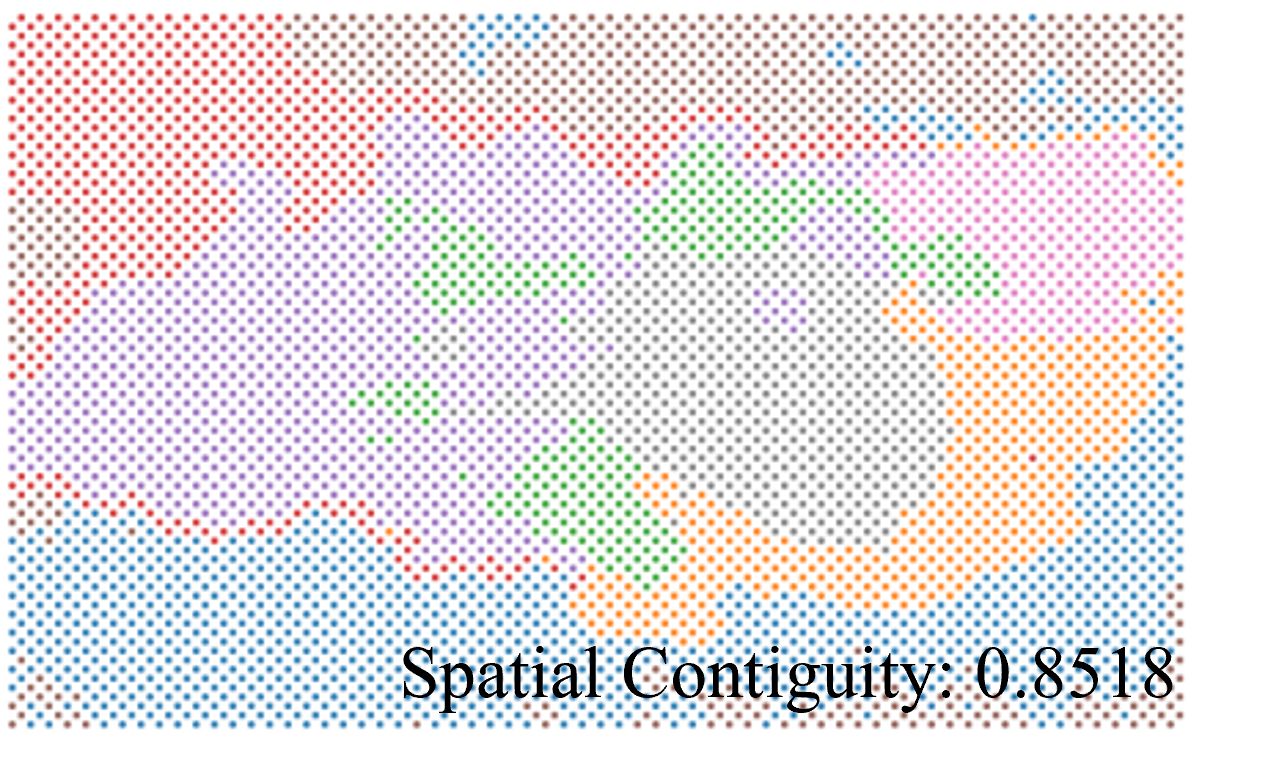}
    \hfill
    \caption{Spatial Leiden clusters for Patient~1 pre-treatment (left) and post-treatment (right)}
    \label{fig:spatial_clusters}
\end{figure}

\begin{figure}[!tbp]
    \centering
    \includegraphics[width=0.47\columnwidth]{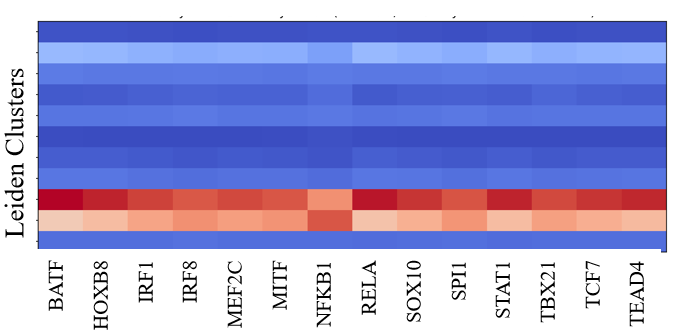}
    \hfill
    \includegraphics[width=0.52\columnwidth]{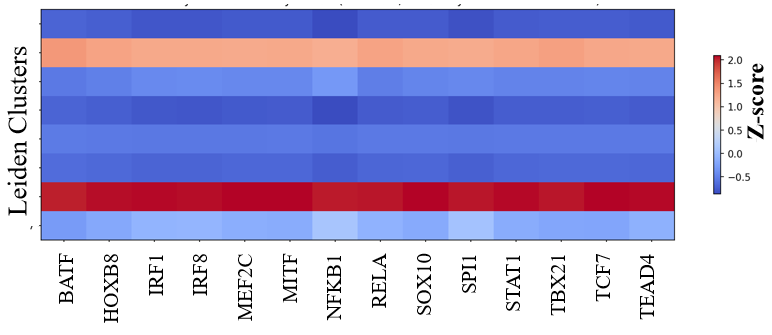}
    \caption{TF program activity by Leiden cluster for Patient~1.}
    \label{fig:tf_program_heatmaps}
\end{figure}

\begin{figure}[!bt]
    \centering
    \includegraphics[width=0.48\linewidth]{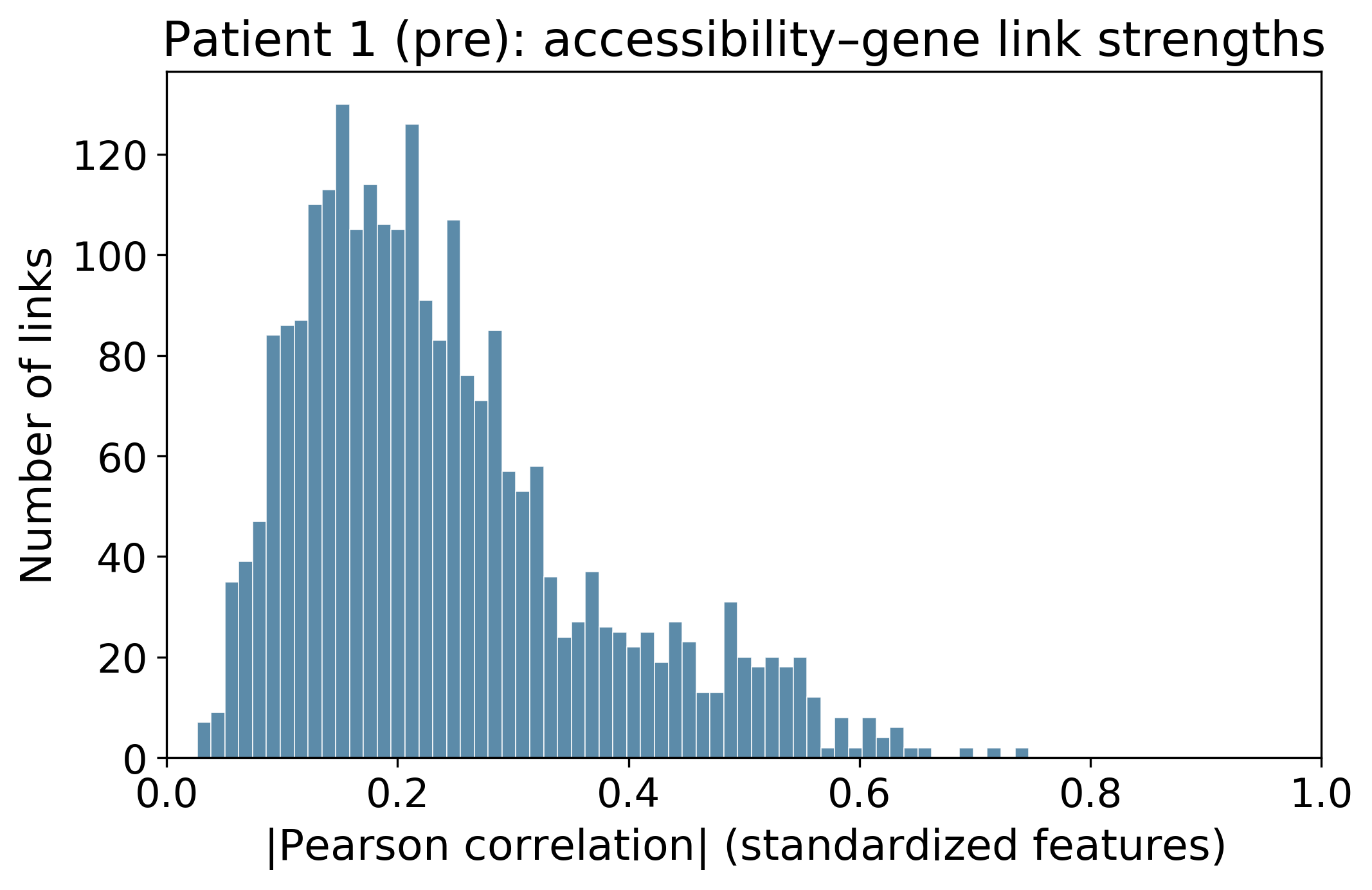}
    \hfill
    \includegraphics[width=0.48\linewidth]{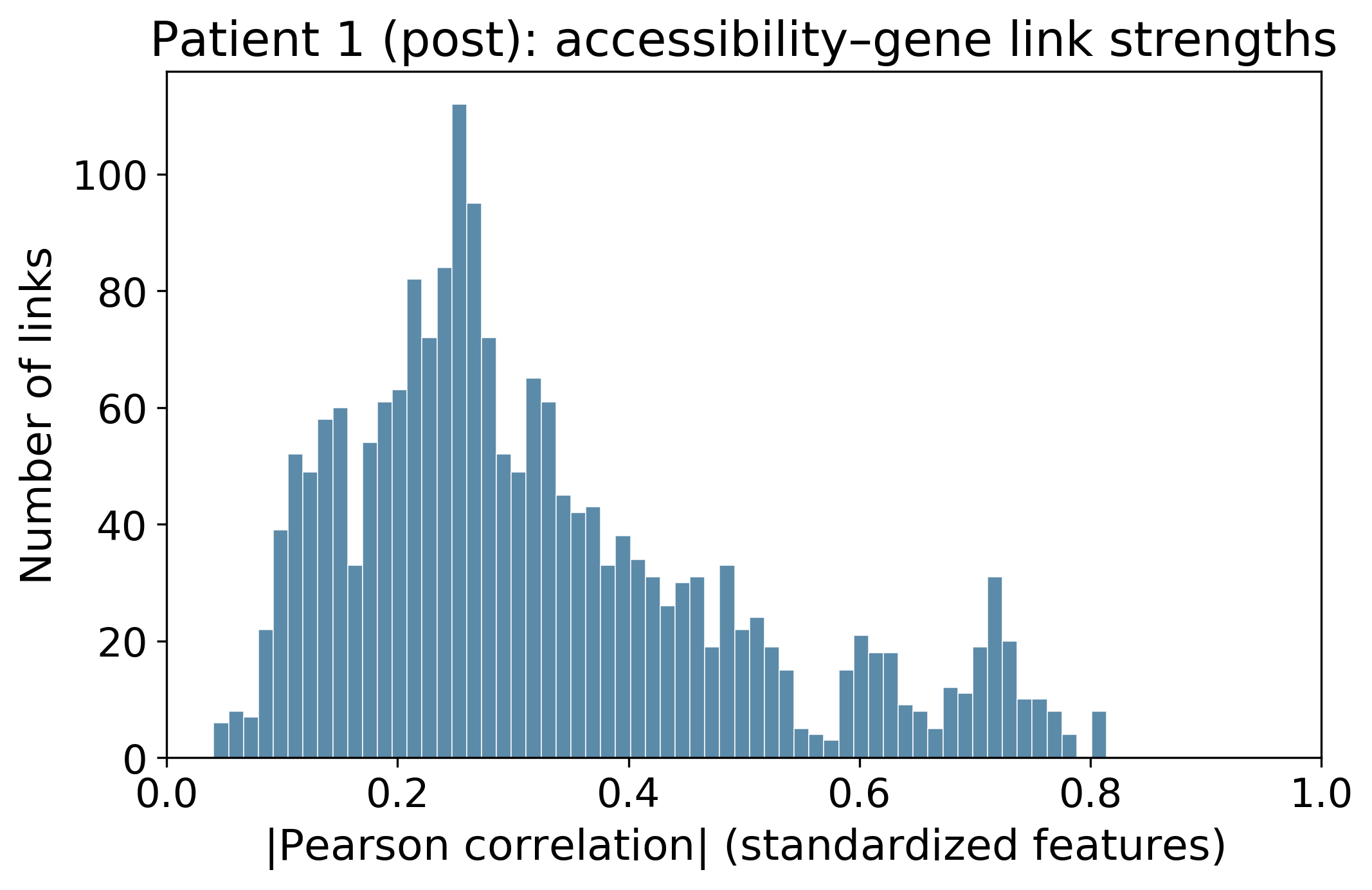}
    \caption{Accessibility--gene link strength distributions for Patient~1.}
    \label{fig:accessibility_gene_links}
\end{figure}

\subsection{Patient-level behavior}
\label{sec:exp_patient}

Paired pre/post-treatment samples show patient-specific shifts. Across the three paired patients, global shift alignment is near zero, and leave-one-patient-out prediction is mixed. We therefore do not claim a consistent cross-patient treatment signature. Instead, these results show that LATTICE captures stable sample-level structure while preserving patient-specific variation.

\subsection{Ablation study}
\label{sec:exp_ablation}

Table~\ref{tab:loss_ablation} separates two questions from the modality ladder. The ladder asks which input blocks help. The ablation table asks which training choices matter when the full M5 input is fixed.

\begin{table}[!tbp]
\centering
\caption{Training ablations and test-time modality dropout using full M5 inputs. Rows remove spatial regularization, masking, spatial ATAC at evaluation, or CUT\&Tag at evaluation. The final row is the full model.}
\label{tab:loss_ablation}
\resizebox{\linewidth}{!}{
\footnotesize
\begin{tabular}{cccc|ccc}
\toprule
Spatial reg. & Masking & Spatial ATAC at eval. & CUT\&Tag at eval. & Spatial contiguity & Silhouette & MUS \\
\midrule
-- & \checkmark & \checkmark & \checkmark & $0.783 \pm 0.076$ & $0.294 \pm 0.067$ & $0.546$ \\
\checkmark & -- & \checkmark & \checkmark & $0.849 \pm 0.044$ & $0.398 \pm 0.083$ & $0.790$ \\
\checkmark & \checkmark & -- & \checkmark & $0.838 \pm 0.044$ & $\mathbf{0.430 \pm 0.154}$ & $0.765$ \\
\checkmark & \checkmark & \checkmark & -- & $0.839 \pm 0.027$ & $0.417 \pm 0.117$ & $0.769$ \\
\midrule
\checkmark & \checkmark & \checkmark & \checkmark & $\mathbf{0.850 \pm 0.036}$ & $0.417 \pm 0.088$ & $\mathbf{0.803}$ \\
\bottomrule
\end{tabular}
}
\end{table}

\noindent Removing spatial regularization lowers spatial contiguity most, supporting the spatial loss. Removing masking changes RNA-reference agreement but does not improve silhouette, suggesting that masking shapes embedding geometry. Withholding chromatin channels at test time does not destabilize training or inference but the metric changes are not uniformly better. Robustness therefore means operational stability, not guaranteed improvement.

\vspace{-2mm}
\section{Conclusion}

LATTICE is a reproducible graph self-supervised learning framework that harmonizes five spot-aligned assay blocks and combines masked reconstruction, cross-modal alignment, and spatial regularization for multimodal spatial omics integration. Across an 11-sample melanoma cohort, LATTICE showed substantial gains when moving from Visium RNA alone to Visium RNA plus projected scMultiome RNA, while full multimodal inputs produced the strongest spatial coherence and MUS performance. ARI and NMI~\cite{strehl2002cluster} did not increase monotonically with additional epigenomic channels, suggesting that the learned embeddings capture regulatory and chromatin structure beyond transcriptome-only organization. Within the upstream integration stack, Visium RNA with projected scMultiome RNA and projected scMultiome ATAC gene scores already performs well within the SARSIM assay scope~\cite{dwarampudi2026spatially}, while adding ReCAST-derived spatial ATAC and spatial CUT\&Tag further improves spatial structure metrics. Paired and LOPO analyses did not reveal a consistent cross-patient treatment signature, indicating that treatment-associated shifts remain patient specific in the current cohort. Richer multimodal views of tissue can support discovery when they are interpreted with appropriate safeguards, and we hope this line of work contributes careful tooling that complements expert judgment rather than replacing it. Future work includes deeper biological validation, broader external benchmarking, and evaluation against additional fully multimodal spatial baselines as such datasets become available.

\bibliographystyle{unsrtnat}
\bibliography{references}

@article{long2023spatially,
  title={Spatially informed clustering, integration, and deconvolution of spatial transcriptomics with GraphST},
  author={Long, Yahui and Ang, Kok Siong and Li, Mengwei and Chong, Kian Long Kelvin and Sethi, Raman and Zhong, Chengwei and Xu, Hang and Ong, Zhiwei and Sachaphibulkij, Karishma and Chen, Ao and others},
  journal={Nature communications},
  volume={14},
  number={1},
  pages={1155},
  year={2023},
  publisher={Nature Publishing Group UK London}
}

@article{dong2022deciphering,
  title={Deciphering spatial domains from spatially resolved transcriptomics with an adaptive graph attention auto-encoder},
  author={Dong, Kangning and Zhang, Shihua},
  journal={Nature communications},
  volume={13},
  number={1},
  pages={1739},
  year={2022},
  publisher={Nature Publishing Group UK London}
}

@article{hu2021spagcn,
  title={SpaGCN: Integrating gene expression, spatial location and histology to identify spatial domains and spatially variable genes by graph convolutional network},
  author={Hu, Jian and Li, Xiangjie and Coleman, Kyle and Schroeder, Amelia and Ma, Nan and Irwin, David J and Lee, Edward B and Shinohara, Russell T and Li, Mingyao},
  journal={Nature methods},
  volume={18},
  number={11},
  pages={1342--1351},
  year={2021},
  publisher={Nature Publishing Group US New York}
}

@article{yang2025spatial,
  title={Spatial integration of multi-omics single-cell data with SIMO},
  author={Yang, Penghui and Jin, Kaiyu and Yao, Yue and Jin, Lijun and Shao, Xin and Li, Chengyu and Lu, Xiaoyan and Fan, Xiaohui},
  journal={Nature communications},
  volume={16},
  number={1},
  pages={1265},
  year={2025},
  publisher={Nature Publishing Group UK London}
}

@article{zhu2024maxfuse,
  title={MaxFuse enables data integration across weakly linked spatial and single-cell modalities},
  author={Zhu, Bokai and Ma, Zongming},
  journal={Nature biotechnology},
  volume={42},
  number={7},
  pages={1036--1037},
  year={2024},
  publisher={Nature Portfolio}
}

@article{zhang2026robust,
  title={Robust integration and annotation of single-cell and spatial omics data using interpretable gene programs},
  author={Zhang, Yuelei and Ming, Wenxuan and Yu, Bianjiong and Wang, Lele and Lu, Kaiyan and Xu, Lei and Ni, Yanhong and Deng, Runzhi and Chen, Dijun},
  journal={Cell Genomics},
  volume={6},
  number={4},
  year={2026},
  publisher={Elsevier}
}

@article{staahl2016visualization,
  title={Visualization and analysis of gene expression in tissue sections by spatial transcriptomics},
  author={St{\aa}hl, Patrik L and Salm{\'e}n, Fredrik and Vickovic, Sanja and Lundmark, Anna and Navarro, Jos{\'e} Fern{\'a}ndez and Magnusson, Jens and Giacomello, Stefania and Asp, Michaela and Westholm, Jakub O and Huss, Mikael and others},
  journal={Science},
  volume={353},
  number={6294},
  pages={78--82},
  year={2016},
  publisher={American Association for the Advancement of Science}
}

@article{kaya2019cut,
  title={CUT\&Tag for efficient epigenomic profiling of small samples and single cells},
  author={Kaya-Okur, Hatice S and Wu, Steven J and Codomo, Christine A and Pledger, Erica S and Bryson, Terri D and Henikoff, Jorja G and Ahmad, Kami and Henikoff, Steven},
  journal={Nature communications},
  volume={10},
  number={1},
  pages={1930},
  year={2019},
  publisher={Nature Publishing Group UK London}
}

@article{traag2019louvain,
  title={From Louvain to Leiden: guaranteeing well-connected communities},
  author={Traag, Vincent A and Waltman, Ludo and Van Eck, Nees Jan},
  journal={Scientific reports},
  volume={9},
  number={1},
  pages={5233},
  year={2019},
  publisher={Nature Publishing Group UK London}
}

@article{mcinnes2018umap,
  title={Umap: Uniform manifold approximation and projection for dimension reduction},
  author={McInnes, Leland and Healy, John and Melville, James},
  journal={arXiv preprint arXiv:1802.03426},
  year={2018}
}

@article{strehl2002cluster,
  title={Cluster ensembles---a knowledge reuse framework for combining multiple partitions},
  author={Strehl, Alexander and Ghosh, Joydeep},
  journal={Journal of machine learning research},
  volume={3},
  number={Dec},
  pages={583--617},
  year={2002}
}

@article{rousseeuw1987silhouettes,
  title={Silhouettes: a graphical aid to the interpretation and validation of cluster analysis},
  author={Rousseeuw, Peter J},
  journal={Journal of computational and applied mathematics},
  volume={20},
  pages={53--65},
  year={1987},
  publisher={Elsevier}
}

@article{paszke2019pytorch,
  title={Pytorch: An imperative style, high-performance deep learning library},
  author={Paszke, Adam and Gross, Sam and Massa, Francisco and Lerer, Adam and Bradbury, James and Chanan, Gregory and Killeen, Trevor and Lin, Zeming and Gimelshein, Natalia and Antiga, Luca and others},
  journal={Advances in neural information processing systems},
  volume={32},
  year={2019}
}

@article{fey2019fast,
  title={Fast graph representation learning with PyTorch Geometric},
  author={Fey, Matthias and Lenssen, Jan Eric},
  journal={arXiv preprint arXiv:1903.02428},
  year={2019}
}

@article{loshchilov2017decoupled,
  title={Decoupled weight decay regularization},
  author={Loshchilov, Ilya and Hutter, Frank},
  journal={arXiv preprint arXiv:1711.05101},
  year={2017}
}

@inproceedings{gutmann2010noise,
  title={Noise-contrastive estimation: A new estimation principle for unnormalized statistical models},
  author={Gutmann, Michael and Hyv{\"a}rinen, Aapo},
  booktitle={Proceedings of the thirteenth international conference on artificial intelligence and statistics},
  pages={297--304},
  year={2010},
  organization={JMLR Workshop and Conference Proceedings}
}

@article{dwarampudi2026spatially,
  title={Spatially Anchored Regulatory State Inference in Melanoma},
  author={Dwarampudi, Jagan Mohan Reddy and Kochat, Veena and Satpati, Suresh and Mahmud, Md Ishtyaq and Anzum, Humaira and Wani, Khalida and Lazar, Alexander and Saw, Ajay Kumar and Malke, Jared and Nguyen, Hien V and others},
  journal={bioRxiv},
  year={2026},
  publisher={Cold Spring Harbor Laboratory}
}

@misc{tenx_visium_hd,
  title        = {HD Spatial Gene Expression},
  author       = {{10x Genomics}},
  year         = {2024},
  url          = {https://www.10xgenomics.com/support/spatial-gene-expression-hd},
  urldate      = {2026-05-06},
  note         = {Accessed: 2026-05-06}
}

@misc{tenx_xenium,
  title        = {Xenium In Situ Platform},
  author       = {{10x Genomics}},
  year         = {2024},
  url          = {https://www.10xgenomics.com/platforms/xenium},
  urldate      = {2026-05-06},
  note         = {Accessed: 2026-05-06}
}

@misc{tenx_spaceranger,
  title        = {Space Ranger},
  author       = {{10x Genomics}},
  year         = {2026},
  url          = {https://www.10xgenomics.com/support/software/space-ranger/latest},
  urldate      = {2026-05-06},
  note         = {Accessed: 2026-05-06}
}


\appendix
\setcounter{figure}{0}
\renewcommand{\thefigure}{A\arabic{figure}}

\section{Extended Experimental Details}
\label{sec:appendix_setup}

\subsection{Datasets and Preprocessing}

The successful batch uses 11 private collaborator sample runs (54{,}912 total spots), each with five modality blocks aligned on the strict SARSIM-anchor five-way gene intersection. In the reported configuration, highly variable gene (HVG) filtering is configured but disabled (\texttt{apply\_hvg\_filter=false}). Modality preprocessing uses per-modality \texttt{log1p\_then\_zscore} before concatenation.

\subsection{Evaluation Metrics\label{sec:a2}}

\paragraph{Multimodal Utility Score (MUS).}
Let \(\mathrm{SpotCut}_v\) denote graph-edge spatial contiguity on the spot kNN graph, \(\mathrm{Silhouette}_v\) the mean cosine silhouette~\cite{rousseeuw1987silhouettes} of embeddings with respect to Leiden clusters~\cite{traag2019louvain}, \(\mathrm{BioKNN}_v\) the same-cluster neighbor consistency from single-sample biological diagnostics, and \(\mathrm{BioJaccard}_v\) the mean Jaccard overlap between embedding kNNs and spatial kNNs. Each term is min--max normalized across all methods and LATTICE modality-ladder rows in the cohort evaluation pool before averaging:
\[
\mathrm{MUS}_v=\tfrac{1}{4}\left(
\mathrm{SpotCut}^{\mathrm{norm}}_v+
\mathrm{Silhouette}^{\mathrm{norm}}_v+
\mathrm{BioKNN}^{\mathrm{norm}}_v+
\mathrm{BioJaccard}^{\mathrm{norm}}_v
\right).
\]
The resulting MUS values rank rows in Table~\ref{tab:cohort_benchmark_mus}. External methods use the same graph construction and $K$-matching procedure as LATTICE with a tractable Leiden~\cite{traag2019louvain} resolution sweep. Additional implementation details accompany the released artifacts.

\paragraph{Marker Gene Enrichment Score.}
To quantify biological consistency, we define a marker gene enrichment score for each cluster. Let $G_c$ denote a set of known marker genes for cluster $c$. For each gene $g \in G_c$, we compute:
\[
\Delta_{g,c} = \bar{x}_{g}^{\text{in}} - \bar{x}_{g}^{\text{out}}.
\]

The cluster score is:
\[
\text{Score}_c = \frac{1}{|G_c|} \sum_{g \in G_c} \Delta_{g,c},
\]
and the overall score:
\[
\text{Marker Score} = \frac{1}{C} \sum_{c=1}^{C} \text{Score}_c.
\]

Higher values indicate stronger enrichment of marker genes within predicted clusters.

\section{Extended Ablation Studies}
\label{sec:appendix_ablation}

\subsection{Relation to Main-Text Ablations}

Cohort-scale evidence for \emph{which} modalities are included appears in the M1 through M5 ladder (Table~\ref{tab:cohort_benchmark_mus}). Evidence for \emph{training} choices and \emph{test-time} withholding of chromatin blocks while keeping the full M5 tensor definition appears in Table~\ref{tab:loss_ablation}. Together these tables cover RNA-only through full multimodal inputs, spatial regularization and masking, and evaluation when spatial ATAC or spatial CUT\&Tag are withheld at test time. We do not report an additional appendix grid beyond these cohort summaries.

\subsection{Effect of Modality Combinations}

We evaluate the configured modality-combination variants exported by the analysis matrix. Improvements are strongest for Visium RNA alone to Visium RNA plus projected scMultiome RNA, while additional modalities show metric-dependent trade-offs (for example, better spatial contiguity without uniformly higher ARI/NMI).

\subsection{Effect of Graph Modeling}

Training keeps the spot $k$-NN graph for graph convolution while optionally applying explicit spatial smoothness in the objective (Section~\ref{sec:exp_ablation}). Removing spatial regularization lowers cohort spatial contiguity and RNA-reference agreement relative to LATTICE M5 (Table~\ref{tab:loss_ablation}), supporting explicit spatial coupling beyond graph propagation alone.

\section{Robustness to Missing Modalities}
\label{sec:appendix_robustness}

\subsection{Test-Time Modality Withholding}

We follow the same test-time withholding protocol as Section~\ref{sec:exp_ablation}: chromatin blocks may be omitted at evaluation while the model trained on full M5 inputs is held fixed. Cohort-level outcomes appear in Table~\ref{tab:loss_ablation} (rows with spatial ATAC or CUT\&Tag absent at evaluation). Runs complete without failures and remain numerically stable, but metric changes are mixed, so we do not claim uniform superiority under all missing-modality patterns.

\subsection{Comparison to Naive Handling}

The robustness discussion in the main text applies here as well: graceful execution is observed, without monotonic gains across all metrics when modalities are missing at test time.

\section{Hyperparameter and Implementation Summary}
\label{sec:appendix_sensitivity}

Fixed hyperparameters for the reported batch (architecture, loss weights, optimization, graph neighborhood size, masking ratio, and alignment temperature) are listed in Appendix~\ref{sec:implementation}. We do not include quantitative one-dimensional sensitivity curves in this submission. Such sweeps can accompany extended reproducibility materials if needed.

\section{Patient-Level Analysis}
\label{sec:appendix_patient}

\subsection{Per-Patient Performance}

Per-sample metrics and paired pre/post-treatment summaries appear in the main text (Section~\ref{sec:exp_patient}) alongside cohort-level tables.

\subsection{Cross-Patient Generalization}

Leave-one-patient-out analyses in Section~\ref{sec:exp_patient} show heterogeneous pre/post-treatment prediction across patients. We therefore describe current evidence as supporting stable per-sample embeddings without consistent cross-patient generalization for treatment prediction.

\section{ReCAST Pipeline Notes}
\label{sec:appendix_recast}

ReCAST is the internal preprocessing and harmonization stage used to prepare multimodal tensors before SARSIM and LATTICE. It is an internal engineering pipeline and is included here for reproducibility of this submission's data flow. In our stack, ReCAST is responsible for ingesting raw per-sample Visium and scMultiome exports, applying sample-level quality checks, and producing aligned matrices and manifests consumed by downstream stages.

\paragraph{Scope and outputs.}
At a high level, ReCAST performs three functions. First, it standardizes modality-level inputs into a common sample schema (spot identifiers, feature indices, and QC metadata). Second, it runs cross-modality gene and feature harmonization, including consistency checks needed for downstream five-block concatenation. Third, it exports machine-readable manifests (sample registry, modality availability, and preprocessing parameters) that are versioned with each run so later stages can be replayed deterministically.

\paragraph{Role in this paper.}
For this study, ReCAST provides the harmonized spatial ATAC and spatial CUT\&Tag blocks (reported as ReCAST CGMC in the main text), as well as quality-controlled sample manifests used to define the 14-source/11-successful cohort. SARSIM~\cite{dwarampudi2026spatially} then operates on the harmonized cohort to produce spatially anchored regulatory projections and overlap-gene artifacts. LATTICE consumes those exported blocks and artifacts to train and evaluate the embedding model. This staged organization separates engineering harmonization (ReCAST), regulatory projection (SARSIM), and representation learning (LATTICE), and keeps each component's scope explicit during review. We do not present ReCAST as a scientific baseline or claim methodological novelty for it here. Its purpose in this manuscript is operational reproducibility of the multimodal tensor assembly. Camera-ready materials can include a public code release and fuller implementation documentation for this stage.

\section{Disclosures for reproducibility and responsible research}
\label{sec:appendix_disclosures}
This appendix states how we share artifacts, what ran on the cluster, and how human data were governed. Model and training specifics appear in Appendix~\ref{sec:implementation}.

\subsection{Supplementary code and data availability}
\label{sec:appendix_code_data}
We include anonymized code, Slurm driver scripts, pinned dependency manifests, and run snapshots as supplementary material. The cohort tensors are de-identified clinical biospecimen-derived profiles under a collaborator institution's proprietary agreement, cannot be redistributed publicly, and have no public five-modality substitute at this lattice resolution.

\subsection{Compute scheduling and resources}
\label{sec:appendix_compute}
All reported LATTICE training used CPU execution through manifests that set \texttt{device=cpu}, so reproducing the batch at Visium density does not require a GPU.
In practice each end-to-end per-sample LATTICE job occupied one cluster node provisioned with sixteen CPU cores and 128~GB RAM, and the joint multisample follow-up used the same one-node, sixteen-core, 128~GB footprint.
ReCAST harmonization for spatial epigenomics ran on single nodes with eight cores and 64~GB RAM.
Peak resident set size stayed below those RAM ceilings for completed jobs.
Archived cluster logs that bracket a full per-sample run from harmonization through graph SSL, embedding export, and routine clustering or plot steps show roughly eight minutes of wall clock under light queue contention, which we take as the indicative per-sample duration for planning.
Across the retained cohort, early stopping spans 36--63 epochs with a mean of 45.2, so runs that exit earlier skew shorter while queue wait or heavier diagnostics skew longer.
Supplementary logs retain timestamps and optional accounting exports for readers who need finer per-job estimates.

\subsection{Research ethics and data agreements}
\label{sec:appendix_ethics}
Ethics and IRB oversight for human biospecimens are handled by the collaborating clinical institution. A data transfer agreement is in place between that institution and the authors' site.
For double-blind review, identifying protocol numbers and committee names are withheld. They can be restored in the camera-ready version together with any journal-required ethics boilerplate.

\section{Implementation Details}
\label{sec:implementation}

LATTICE was implemented in Python using \texttt{PyTorch}~\cite{paszke2019pytorch} and \texttt{PyTorch Geometric}~\cite{fey2019fast}. The pipeline consists of a multimodal integration followed by graph-based self-supervised representation learning. In the integration stage, spatial transcriptomic measurements are combined with projected single-cell and epigenomic modalities to construct a concatenated spot-by-feature matrix. All downstream learning is performed at the level of \emph{spatial capture spots}, consistent with the resolution of the available data.

\paragraph{Third-party software and licenses.}
\label{sec:appendix_licenses}
We additionally rely on Scanpy-style tooling where applicable and on 10x Genomics Space Ranger~\cite{tenx_spaceranger} for RNA-derived reference clusters. Each third-party component remains governed by its upstream license (BSD-style terms for PyTorch, MIT terms for PyTorch Geometric, and commercial terms for Space Ranger). The supplementary bundle ships pinned dependency files (\texttt{environment.yml} and/or \texttt{requirements.txt}) so recipients can verify versions and license compatibility before reuse.

\paragraph{Input features.}
For each sample, we construct a multimodal feature matrix
\[
X = [X^{(1)}, X^{(2)}, \dots, X^{(M)}] \in \mathbb{R}^{N \times D},
\]
where $N$ is the number of spots, $M$ is the number of modality blocks, and $D=\sum_{m=1}^{M} d_m$ is the total feature dimension. In our experiments, $M=5$ with modality blocks: Visium RNA, spatial ATAC (CGMC prediction), spatial CUT\&Tag (CGMC prediction), projected scMultiome RNA, and projected scMultiome ATAC (gene scores). Harmonization uses a strict five-way gene intersection anchored on SARSIM \texttt{overlap\_genes.txt}, and each modality block is standardized with \texttt{log1p\_then\_zscore} before concatenation. Because the final intersected gene set is sample-specific, $D=5G$ varies across runs, and in the successful batch, $G \in [129,322]$.

\paragraph{Graph construction.}
We build a kNN spatial graph from Visium array coordinates with $k=6$ and union symmetrization. Radius-graph mode is implemented in configuration validation but was not used in the reported runs. For $N=4{,}992$ spots per sample, the expected directed edge count before deduplication is approximately $N \times k \approx 29{,}952$.

\paragraph{Encoder architecture.}
Each modality block is first projected into a shared hidden space using a modality-specific linear adapter. Let $d_h$ denote the hidden dimension. We set
\[
d_h = 128.
\]
The projected modality representations are fused by modality-aware mean pooling (weighted by observed modality presence per spot). The fused representation is then passed through a stack of graph encoder layers. In the current implementation, the encoder uses \textbf{TransformerConv} layers with:
\begin{itemize}
    \item number of graph layers: $3$
    \item number of attention heads: $4$
    \item hidden dimension per layer: $128$
    \item dropout rate: $0.1$
    \item activation function: \textbf{ReLU}
    \item normalization: \textbf{LayerNorm}
\end{itemize}
The final embedding dimension is
\[
d = 128.
\]

\paragraph{Decoder architecture.}
A feedforward decoder reconstructs the full multimodal input from the learned embedding. The decoder maps from $\mathbb{R}^{d}$ back to $\mathbb{R}^{D}$ using one hidden layer of width $2d$ (i.e., $256$ units when $d=128$) followed by a linear output layer.

\paragraph{Masked reconstruction.}
During training, we apply random modality-block masking at the feature level. The masking ratio is set to
\[
\rho = 0.15.
\]
Let $\Omega \in \{0,1\}^{N \times D}$ denote the binary reconstruction mask. Reconstruction error is computed only on masked entries using the \textbf{Huber loss} with parameter
\[
\delta = 1.0.
\]

\paragraph{Cross-modal alignment.}
Cross-modal alignment uses a noise-contrastive estimation (NCE)-style objective~\cite{gutmann2010noise} applied to selected modality pairs. In the reported configuration, modality indices zero and one correspond to Visium RNA and spatial ATAC latent branches after graph refinement. Each modality block is passed through a projection head of output dimension
\[
d_c = 64.
\]
The projection head uses a two-layer MLP with hidden width $64$. The temperature parameter in the alignment loss is set to
\[
\tau = 0.1.
\]

\paragraph{Spatial regularization.}
To enforce spatial consistency, we apply graph Laplacian smoothness in the reported runs:
\[
\sum_{(i,j)\in E} w_{ij}\|z_i-z_j\|^2,
\]
with edge coupling induced by the constructed kNN graph. An alternative neighbor-based alignment regularizer is implemented but was not used for the successful batch.

\paragraph{Overall objective.}
The final training loss is
\[
\mathcal{L} = \lambda_1 \mathcal{L}_{\mathrm{rec}} + \lambda_2 \mathcal{L}_{\mathrm{align}} + \lambda_3 \mathcal{L}_{\mathrm{spatial}},
\]
with loss weights:
\[
\lambda_1 = 1.0, \qquad
\lambda_2 = 0.5, \qquad
\lambda_3 = 0.1.
\]

\paragraph{Optimization.}
We optimize the model using AdamW~\cite{loshchilov2017decoupled} with:
\begin{itemize}
    \item learning rate: $1\times10^{-3}$
    \item weight decay: $1\times10^{-4}$
    \item batch size: full graph
    \item number of epochs: up to $100$ (early stopping enabled)
    \item early stopping patience: $20$
    \item gradient clipping norm: $1.0$
\end{itemize}
Training uses full-graph optimization (neighbor sampling is not enabled in the current implementation). Automatic mixed precision is disabled. The validation fraction is set to
\[
0.1,
\]
with train/validation masks generated at the node level.

\paragraph{Reproducibility.}
Single-sample runs use \texttt{global\_seed=42} (propagated into graph-SSL training), with deterministic cuDNN flags enabled when CUDA is used and strict deterministic algorithms disabled by default. Joint multisample analysis uses seed $7$ and reports stability summaries over 11 analysis seeds \{7, 11, 19, 23, 29, 31, 37, 41, 43, 47, 53\}. Model checkpoints, training logs, config snapshots, and embeddings are saved for all runs.

\paragraph{Hardware and runtime.}
Training used CPU-only manifests (\texttt{device=cpu}). Resource footprints and indicative wall times appear in Appendix~\ref{sec:appendix_compute}. Supplementary Slurm logs and optional \texttt{sacct} exports provide per-job timestamps and peak RSS where captured by the scheduler.

\paragraph{Downstream analysis.}
After training, embeddings are clustered by Leiden~\cite{traag2019louvain} with target cluster count $K$ imported from SARSIM clustering metadata and achieved via a resolution sweep. Visualizations are generated with PCA and UMAP~\cite{mcinnes2018umap} (UMAP fallback to PCA if \texttt{umap-learn} is unavailable). Predicted spot clusters are compared with Space Ranger clustering labels using ARI and NMI~\cite{strehl2002cluster} when overlap is available.
\section{Theoretical Analysis of LATTICE}
\label{sec:theory}

In this appendix, we provide theoretical insights into the behavior of the LATTICE objective. We show that the combination of masked reconstruction, cross-modal alignment, and spatial regularization yields embeddings that are (i) spatially smooth, (ii) informative with respect to input modalities, and (iii) aligned across modalities.

\subsection{Preliminaries}

Let $G = (V,E)$ denote the spatial graph with weighted adjacency matrix $W = [w_{ij}]$. Let $Z \in \mathbb{R}^{N \times d}$ denote the learned embeddings, and $X = [X^{(1)}, \dots, X^{(M)}]$ the multimodal input features.

The LATTICE objective is:
\begin{equation}
\mathcal{L} = \lambda_1 \mathcal{L}_{\mathrm{rec}} + \lambda_2 \mathcal{L}_{\mathrm{align}} + \lambda_3 \mathcal{L}_{\mathrm{spatial}}.
\end{equation}

We analyze each component in turn.

\subsection{Spatial Smoothness}

\begin{lemma}[Spatial Smoothness]
Minimizing the spatial regularization term
\begin{equation}
\mathcal{L}_{\mathrm{spatial}} = \sum_{(i,j) \in E} w_{ij} \|z_i - z_j\|^2
\end{equation}
encourages the learned embeddings $Z$ to be smooth over the graph, such that neighboring nodes have similar representations.
\end{lemma}

\textit{Proof sketch.}
The spatial loss can be written in matrix form as:
\begin{equation}
\mathcal{L}_{\mathrm{spatial}} = \mathrm{Tr}(Z^\top L Z),
\end{equation}
where $L = D - W$ is the graph Laplacian and $D$ is the degree matrix. Minimizing $\mathrm{Tr}(Z^\top L Z)$ penalizes high-frequency variations over the graph, forcing $z_i \approx z_j$ for strongly connected nodes. This is a standard result in spectral graph theory, implying smoothness of embeddings over the spatial graph.
\hfill $\square$

\subsection{Reconstruction Consistency}

\begin{lemma}[Reconstruction Consistency]
Assuming sufficient model capacity, minimizing the masked reconstruction loss
\begin{equation}
\mathcal{L}_{\mathrm{rec}} = \frac{1}{|\Omega|} \sum_{i,j} \Omega_{ij} (X_{ij} - \hat{X}_{ij})^2
\end{equation}
ensures that the embedding $Z$ retains sufficient information to reconstruct the multimodal input features.
\end{lemma}

\textit{Proof sketch.}
The reconstruction objective corresponds to training an autoencoder under partial observation. Under standard assumptions on universal function approximation, there exists an encoder-decoder pair $(f_\theta, g_\phi)$ such that $\hat{X} \approx X$. Therefore, minimizing $\mathcal{L}_{\mathrm{rec}}$ forces $Z = f_\theta(X)$ to preserve the information content of $X$, up to the capacity of the model.
\hfill $\square$

\subsection{Cross-Modal Alignment}

\begin{lemma}[Cross-Modal Alignment]
The alignment loss
\begin{equation}
\mathcal{L}_{\mathrm{align}} = - \sum_{i=1}^{N} 
\log \frac{\exp(\mathrm{sim}(h_i^{(a)}, h_i^{(b)})/\tau)}
{\sum_{j=1}^{N} \exp(\mathrm{sim}(h_i^{(a)}, h_j^{(b)})/\tau)}
\end{equation}
encourages representations of the same spatial location across modalities to be aligned, while separating representations of different locations.
\end{lemma}

\textit{Proof sketch.}
The objective follows an NCE-style formulation~\cite{gutmann2010noise}. It maximizes similarity between positive pairs $(h_i^{(a)}, h_i^{(b)})$ while minimizing similarity with negative pairs $(h_i^{(a)}, h_j^{(b)})$ for $j \neq i$, which tightens a mutual-information bound between modality-specific representations under standard assumptions. Thus embeddings for the same spatial location become aligned across modalities.
\hfill $\square$

\subsection{Joint Objective Behavior}

\begin{theorem}[Structured Multimodal Embedding]
Let $Z^\star$ be an optimal solution of the LATTICE objective. Then $Z^\star$ simultaneously:
\begin{enumerate}
    \item preserves multimodal feature information,
    \item aligns modality-specific representations,
    \item enforces spatial smoothness over the graph.
\end{enumerate}
\end{theorem}

\textit{Proof sketch.}
The result follows directly from the additive structure of the objective. The reconstruction term enforces information preservation (Lemma 2), the alignment term enforces cross-modal agreement (Lemma 3), and the spatial term enforces smoothness (Lemma 1). Since each term is minimized jointly, the optimal embedding $Z^\star$ must balance all three properties, yielding a representation that captures both spatial structure and multimodal relationships.
\hfill $\square$

\subsection{Interpretation}

The LATTICE objective can be interpreted as learning embeddings that lie on a low-dimensional manifold structured by both spatial proximity and multimodal consistency. The spatial term enforces local smoothness, while the alignment term relates different modality-specific views of the same location. Together, these objectives produce representations that are robust to noise and capable of integrating heterogeneous molecular signals.


\end{document}